\documentclass[10pt,twocolumn,letterpaper]{article}

\usepackage[pagenumbers]{cvpr} 

\usepackage[dvipsnames]{xcolor}

\definecolor{cvprblue}{rgb}{0.21,0.49,0.74}
\usepackage[pagebackref,breaklinks,colorlinks,citecolor=cvprblue]{hyperref}
\usepackage{multirow}
\usepackage{multicol}
\usepackage{threeparttable}

\usepackage[accsupp]{axessibility}

\usepackage{amsthm}
\newtheorem{theorem}{Theorem}
\newtheorem{definition}{Definition}

\def\paperID{4735} 
\def\confName{CVPR}
\def\confYear{2024}

\title{SpikingResformer: Bridging ResNet and Vision Transformer in Spiking Neural Networks}

\author{Xinyu Shi\textsuperscript{1,2}, Zecheng Hao\textsuperscript{2}, Zhaofei Yu\textsuperscript{1,2}\thanks{Corresponding author: yuzf12@pku.edu.cn}\\
\textsuperscript{1} Institute for Artificial Intelligence, Peking University\\
\textsuperscript{2} School of Computer Science, Peking University\\
}

\begin{document}
\maketitle
\begin{abstract}
The remarkable success of Vision Transformers in Artificial Neural Networks (ANNs) has led to a growing interest in incorporating the self-attention mechanism and transformer-based architecture into Spiking Neural Networks (SNNs).
While existing methods propose spiking self-attention mechanisms that are compatible with SNNs, they lack reasonable scaling methods, and the overall architectures proposed by these methods suffer from a bottleneck in effectively extracting local features.
To address these challenges, we propose a novel spiking self-attention mechanism named Dual Spike Self-Attention (DSSA) with a reasonable scaling method.
Based on DSSA, we propose a novel spiking Vision Transformer architecture called SpikingResformer, which combines the ResNet-based multi-stage architecture with our proposed DSSA to improve both performance and energy efficiency while reducing parameters.
Experimental results show that SpikingResformer achieves higher accuracy with fewer parameters and lower energy consumption than other spiking Vision Transformer counterparts.
Notably, our SpikingResformer-L achieves 79.40\% top-1 accuracy on ImageNet with 4 time-steps, which is the state-of-the-art result in the SNN field.
Codes are available at \href{https://github.com/xyshi2000/SpikingResformer}{https://github.com/xyshi2000/SpikingResformer}
\end{abstract}
\section{Introduction}

Spiking Neural Networks (SNNs), considered as the third generation of Artificial Neural Networks (ANNs)~\cite{maass1997networks}, have garnered significant attention due to their notable advantages such as low power consumption, biological plausibility, and event-driven characteristics that are compatible with neuromorphic hardware.
In comparison to ANNs, SNNs exhibit an energy-saving advantage when deployed on neuromorphic hardware~\cite{furber2014spinnaker, merolla2014million,pei2019towards}, and have become popular in the field of neuromorphic computing in recent years~\cite{schuman2017survey}.
However, the performance of existing SNNs still lags behind that of ANNs, particularly on challenging vision tasks, which limits further application of SNNs.

\begin{figure}[!t]
    \centering
    \includegraphics[width=\linewidth]{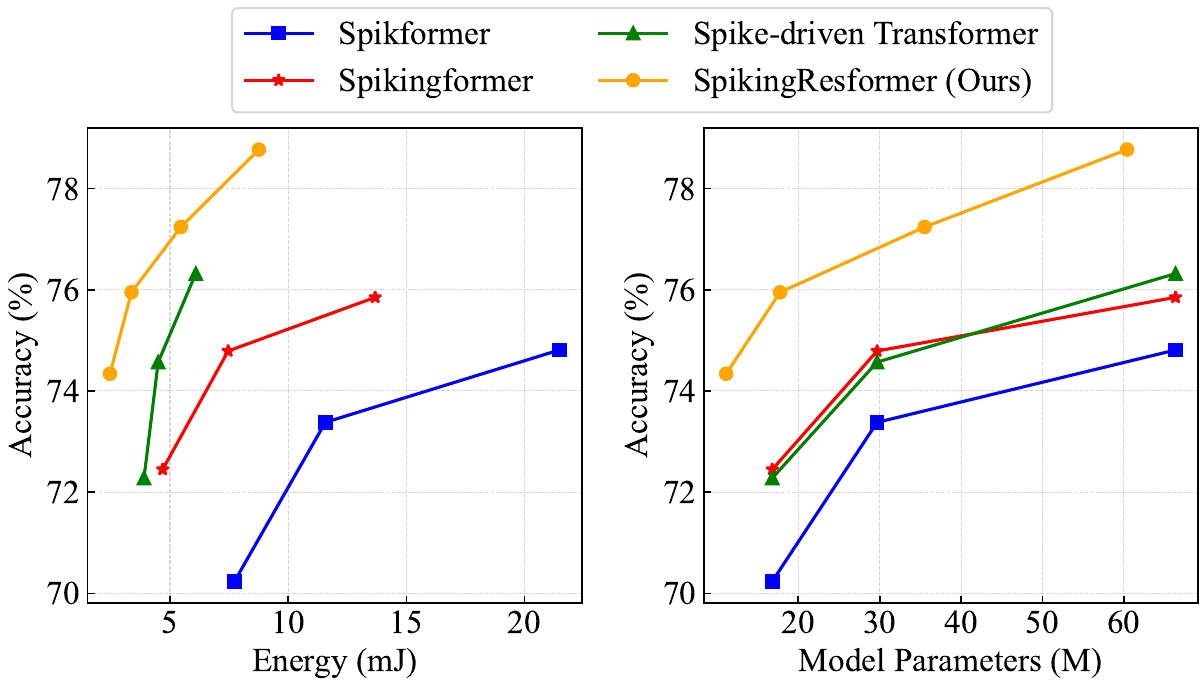}
    \caption{Comparison of Top-1 accuracy on ImageNet with respect to energy consumption per image for inference (left) and the number of parameters (right). The input size is 224$\times$224.}
    \label{fig:comparison}
\end{figure}

Researchers make great efforts to enhance the performance of SNNs.
Recently, inspired by the remarkable achievements of vision transformers~\cite{dosovitskiy2021an,liu2021swin} in ANNs, some attempts have been made to introduce transformer-based architecture into SNNs.
The main challenge in incorporating transformer structures into SNNs lies in designing a self-attention mechanism suitable for SNNs.
The vanilla self-attention mechanism in ANNs relies on float-point matrix multiplication and softmax operations. However,
these operations involve float-point multiplication, division, and exponentiation, which do not comply with the spike-driven nature of SNNs.
Moreover, commonly used components in ANN Transformers, such as layer normalization (LN) and GELU activation, are not directly applicable in SNNs.
Therefore, the introduction of transformer architectures into SNNs necessitates circumventing these operations and catering to the unique requirements of SNNs.

Some work has attempted to partially preserve floating-point operations~\cite{wang2023sstformer}, or run an SNN and an ANN Transformer in parallel~\cite{zou2023event}.
Although these approaches enhance performance, they do not fully address the incompatibility between vanilla self-attention and SNNs.
There are also works that propose spiking self-attention mechanisms and spiking Vision Transformer architectures, which are entirely based on synaptic operations~\cite{zhou2023spikformer} and fully spike-driven~\cite{zhou2023spikingformer,yao2023spike}.
These works completely resolve the incompatibility of the self-attention mechanism with SNNs, significantly outperforming existing spiking convolutional neural networks. However, these approaches have certain limitations.
Specifically, the spiking transformer architectures of these methods exhibit a bottleneck in extracting local features. They employ a shallow convolutional network before the Transformer encoder to extract local features and reduce the size of feature maps. However, the effectiveness of such a shallow network is limited compared to Transformer encoders.
Replacing this shallow network with a Transformer encoder is not feasible.
This is because their spiking self-attention mechanisms lack reasonable scaling methods, making them suitable only for small feature maps. Designing proper scaling factors for these mechanisms is challenging as the input currents to neurons, which generate self-attention, do not possess simple mean and variance forms. Thus, there is a pressing need to design a new spiking self-attention mechanism with a reasonable scaling method that can effectively handle large feature maps.

To address these problems, we propose a novel spiking self-attention mechanism, named Dual Spike Self-Attention (DSSA).
It produces spiking self-attention via Dual Spike Transformation, which is fully spike-driven,  compatible with SNNs, and eliminates the need for direct spike multiplications.
In addition, we detail the scaling method in DSSA, enabling it to adapt to adapt to feature maps of arbitrary scales.
Building upon DSSA, we propose SpikingResformer, a novel spiking Vision Transformer architecture.
This architecture combines the ResNet-based multi-stage design with our proposed spiking self-attention mechanism.
Experimental results show that our proposed SpikingResformer significantly outperforms the performance of existing spiking Vision Transformers with fewer parameters and less energy consumption.
The main contributions of this paper can be summarized as follows:
\begin{itemize}
\item We propose the Dual Spike Self-Attention (DSSA), a novel spiking self-attention mechanism.
It produces spiking self-attention via Dual Spike Transformation, which is fully spike-driven and compatible with SNNs.
\item We detail the scaling factors employed in DSSA, enabling DSSA to handle feature maps of arbitrary scales.
\item We propose the SpikingResformer architecture, which combines the ResNet-based multi-stage architecture with our proposed DSSA.
\item Experimental results show that the proposed SpikingResformer significantly outperforms other spiking Vision Transformer counterparts with fewer parameters and lower energy consumption. Notably, our SpikingResformer-L achieves up to 79.40\% top-1 accuracy on ImageNet.
\end{itemize}

\section{Related Work}

\noindent
\textbf{Spiking Convolutional Neural Networks.}
Spiking convolutional neural networks (SCNNs) have been extensively developed due to the remarkable success of surrogate gradient learning~\cite{zenke2021remarkable,neftci2019surrogate} and are widely used in handling challenging vision tasks, including object recognition~\cite{amir2017low,lan2023efficient}, detection~\cite{kim2020spiking,su2023deep}
, and segmentation~\cite{kirkland2020spikeseg,patel2021spiking}.
To improve the performance of SCNNs on these challenging tasks, researchers have dedicated great efforts to exploring training methods~\cite{lee2020enabling,deng2022temporal,wei2023temporal,guo2023membrane,meng2023towards,zhu2022training,zhu2023exploring,xu2023constructing} and ANN-to-SNN conversion techniques~\cite{sengupta2019going,hu2021spiking,hao2023reducing,li2023unleashing,li2021free,jiang2023unified,wu2021progressive,hu2023fast,bu2023optimal,bu2022optimized,ding2021optimal,deng2021optimal}.
Moreover, many deep spiking convolutional architectures~\cite{sengupta2019going,lee2020enabling,zheng2021going,hu2021spiking,fang2021deep} have been proposed to achieve high performance.
The success of SpikingVGG~\cite{sengupta2019going,lee2020enabling} demonstrates that SCNNs can achieve comparable performance to ANNs in recognition tasks.
SpikingResNet~\cite{zheng2021going,hu2021spiking} further explores SCNNs with residual structure and achieves deeper SCNN with ResNet-based architecture. Moreover, SEW ResNet~\cite{fang2021deep} meticulously analyzes the identity mapping in directly-trained spiking residual networks and successfully trains a 152-layer SCNN directly.
These architectures leverage large-scale SNNs with numerous layers and demonstrate superior performance on various tasks.

\noindent
\textbf{Spiking Vision Transformers.}
Spikformer~\cite{zhou2023spikformer} is the first directly-trained spiking vision transformer with a pure SNN architecture.
It introduces a spiking self-attention mechanism that eliminates multiplication by activating Query, Key, and Value with spiking neurons and replacing softmax with spiking neurons.
In addition, it replaces layer normalization and GELU activation in the Transformer with batch normalization and spiking neurons.
Based on Spikformer, Spikingformer~\cite{zhou2023spikingformer} achieves a purely spike-driven Vision Transformer by modifying the residual connection paradigm.
Spike-driven Transformer~\cite{yao2023spike} proposes a spike-driven self-attention mechanism with linear complexity, effectively reducing energy consumption.
However, all of these efforts employ a shallow convolutional network to pre-extract local information to form a sequence of patches and lack proper scaling methods.

\section{Preliminary}

We describe the dynamics of the Leaky Integrate-and-Fire (LIF) neuron used in this paper by the following discrete-time model:
\begin{align}
v_i[t] &= u_i[t] + \frac1\tau (I_i[t]-(u_i[t]-u_{\rm rest})),\label{eq:spiking neuron model a}\\
s_i[t] &= H(v_i[t] - u_{\rm th}),\label{eq:spiking neuron model b}\\
u_i[t+1] &= s_i[t] u_{\rm rest}+(1-s_i[t])v_i[t].\label{eq:spiking neuron model c}
\end{align}
Eq.~\eqref{eq:spiking neuron model a} describes the charging process.
Here $u_i[t]$ and $v_i[t]$ denote the membrane potential of $i$-th postsynaptic neuron before and after charging. $\tau$ is the membrane time constant. $I_i[t]$ denotes the input current.
In general, $I_i[t] = \sum_j w_{i,j}s_j[t]$, where $s_j[t]\in\{0,1\}$ represents the spike output of the $j$-th presynaptic neuron at time-step $t$, and $w_{i,j}$ represents the weight of the corresponding synaptic connection from neuron $j$ to neuron $i$.
Eq.~\eqref{eq:spiking neuron model b} describes the firing process,
where $H(\cdot)$ is the Heaviside function, $s_i[t]\in\{0,1\}$ is the spike output of the spiking neuron, $u_{\rm th}$ denotes the firing threshold.
Eq.~\eqref{eq:spiking neuron model c} describes the resetting process, with
$u_{\rm rest}$ denoting the resting potential.

For the sake of simplicity and clarity in subsequent sections, we represent the spiking neuron model as follows:
\begin{equation}
    {\bf S} = {\rm SN}({\bf I}),
\end{equation}
where ${\rm SN}(\cdot)$ denotes the spiking neuron layer, omitting the dynamic processes within the neuron, ${\bf I}\in \mathbb{R}^{T\times n}$ is the input current, where $T$ is the time step and $n$ is the number of neurons, ${\bf S}\in \{0,1\}^{T\times n}$ is the corresponding spike output.

\section{Dual Spike Self-Attention}
\label{sec:DSSA}

This section first revisits the vanilla self-attention (VSA) mechanism commonly used in ANNs and analyzes why VSA is not suitable for SNNs. Then, we propose dual spike self-attention (DSSA), specifically designed for compatibility. We further discuss the significance of the scaling factor in DSSA and the spike-driven characteristic of DSSA.

\subsection{Vanilla Self-Attention}
The vanilla self-attention in Transformer~\cite{vaswani2017attention} can be formulated as follows:
\begin{align}
&{\bf Q}={\bf X}{\bf W}_Q,\;{\bf K}={\bf X}{\bf W}_K,\;{\bf V}={\bf X}{\bf W}_V,\;\\
&\mathrm{Attention}({\bf Q}, {\bf K}, {\bf V})
= \mathrm{Softmax}\left(\frac{{\bf QK}^{\rm T}}{\sqrt{d}}\right){\bf V}
\label{eq:vanilla attention}
\end{align}
Here ${\bf Q}, {\bf K}, {\bf V}\in\mathbb{R}^{n\times d}$ denote Query, Key and Value, respectively. $n$ is the number of patches, $d$ is the embedding dimension.
We assume that ${\bf Q}, {\bf K}, {\bf V}$ have the same embedding dimension.
${\bf X}\in\mathbb{R}^{n\times d}$ is the input of self-attention block, ${\bf W}_Q, {\bf W}_K, {\bf W}_V\in\mathbb{R}^{d\times d}$ are the weights of the linear transformations corresponding to ${\bf Q}, {\bf K}, {\bf V}$, respectively.

The vanilla self-attention commonly used in ANNs is not suitable for SNNs due to the following two types of operations involved:
1) the float-point matrix multiplication of $\bf Q$ and $\bf K$, as well as between the attention map and $\bf V$;
2) the softmax function, which contains exponentiation and division operations.
These operations rely on float-point multiplication, division, and exponentiation operations, which are not compatible with the restrictions of SNNs.

\begin{figure*}[!t]
    \centering
    \includegraphics[width=\linewidth]{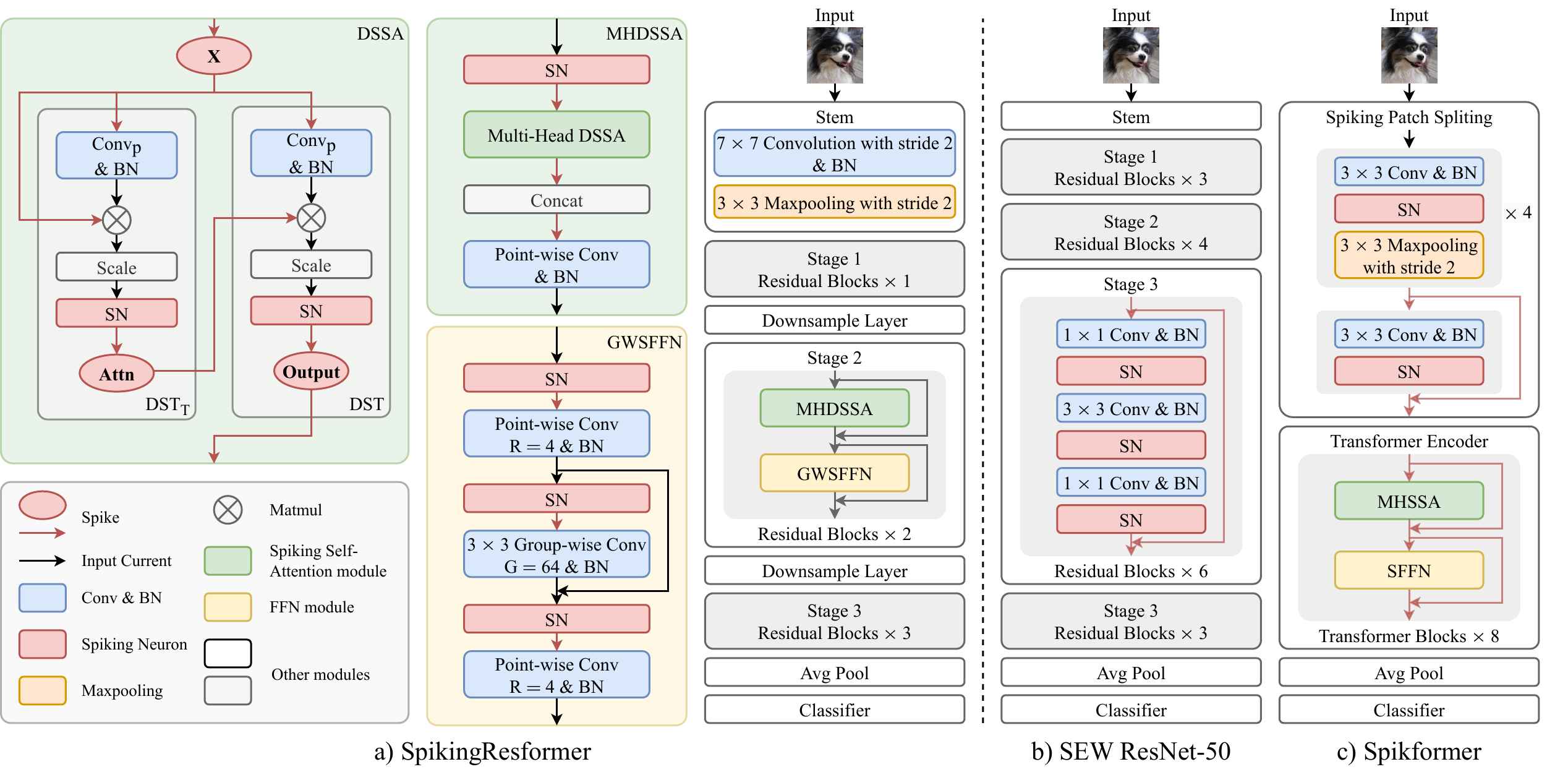}
    \caption{{\bf Left}: Architecture of SpikingResformer and components including Dual Spike Self-Attention (DSSA), Multi-Head DSSA (MHDSSA), and Group-Wise Spiking Feed-Forward Network (GWSFFN). {\bf Right}: Architecture of SEW ResNet-50 and Spikformer.}
    \label{fig:overview}
\end{figure*}

\subsection{Dual Spike Self-Attention}
To introduce the self-attention mechanism into SNNs and efficiently handle the multi-scale feature maps, we propose a novel spiking self-attention mechanism, named Dual Spike Self-Attention (DSSA).
DSSA only utilizes Dual Spike Transformation (DST), thereby eliminating the need for the float-point matrix multiplication and softmax function.
We first define the DST as follows:
\begin{align}
{\rm DST}({\bf X}, {\bf Y}; f(\cdot))&={\bf X}f({\bf Y})={\bf XYW},\label{eq:DST}\\
{\rm DST_T}({\bf X}, {\bf Y}; f(\cdot))&={\bf X}f({\bf Y})^{\rm T}={\bf XW}^{\rm T}{\bf Y}^{\rm T}.
\label{eq:DSTT}
\end{align}
In Eq.~\eqref{eq:DST}, ${\bf X}\in\{0,1\}^{T\times p\times m}$ and ${\bf Y}\in\{0,1\}^{T\times m\times q}$ represent the dual spike inputs. $T$ is the time steps, $p$, $m$, $q$ denote arbitrary dimensions.
Here $f(\cdot)$ is a generalized linear transformation on ${\bf Y}$, with ${\bf W}\in\mathbb{R}^{q\times q}$ denoting its weight matrix. It represents any operation that can be equated to a linear transformation, including convolution, batch normalization (ignoring bias), etc.
A detailed discussion on the equivalence of convolution operations to linear transformations can be found in the supplementary.
Similarly, in Eq.~\eqref{eq:DSTT}, we have ${\bf Y}\in\{0,1\}^{T\times q\times m}$, ${\bf W}\in\mathbb{R}^{m\times m}$.
Since $\bf X$ and $\bf Y$ are both spike matrices, all matrix multiplications are equivalent to the summation of weights. Consequently, the DST avoids floating-point multiplication, making it compatible with SNNs.
Further discussion of the compatibility and the spike-driven characteristic of DST can be found in Sec.~\ref{sec:spike-driven}.
Based on DST, the attention map in DSSA can be formulated as follows:
\begin{align}
{\rm AttnMap}({\bf X})&={\rm SN}({\rm DST_T}({\bf X}, {\bf X};f(\cdot))*c_1),\\
f({\bf X})&={\rm BN}({\rm Conv_p}({\bf X})),
\end{align}
where ${\bf X}\in \{0,1\}^{T\times HW\times d}$ is the spike input, with $H$ and $W$ denoting the spatial height and width of the input, $d$ denoting the embedding dimension. $\rm BN(\cdot)$ refers to the batch normalization layer, $\rm Conv_p(\cdot)$ denotes a $p\times p$ convolution with a stride of $p$, and $c_1$ is the scaling factor.
Since the convolution operation is equivalent to a generalized linear transformation, and batch normalization can be absorbed into the convolution (ignoring bias), ${\rm BN}({\rm Conv_p}(\cdot))$ can be viewed as a generalized linear transformation.
Here we use the $p\times p$ convolution with a stride of $p$ to reduce the spatial size to handle the multi-scale feature map and reduce the computational overhead.
In DSSA, there is no need for the softmax function since the spiking neuron layer inherently generates a binary attention map composed of spikes.
Each spike $s_{ij}$ in this spiking attention map signifies attention between patch $i$ and patch $j$.
We believe that such a spiking attention map is more interpretable compared to the attention map in ANN activated with softmax. With the spiking attention map, the DSSA can be formulated as follows:
\begin{equation}
{\rm DSSA}({\bf X})={\rm SN}({\rm DST}({\rm AttnMap}({\bf X}),{\bf X};f(\cdot))*c_2),
\label{eq:DSSA}
\end{equation}
where $c_2$ is the second scaling factor. Since the form of DSSA is quite different from VSA and existing spiking self-attention mechanisms, we further discuss how DSSA achieves self-attention in the supplementary.

\subsection{Scaling Factors in DSSA}
In the vanilla self-attention mechanism~\cite{vaswani2017attention}, the product of matrices $\mathbf{Q}$ and $\mathbf{K}$ in Eq.~\eqref{eq:vanilla attention} is scaled by a factor of $1/\sqrt{d}$ before applying the softmax operation.
This scaling is necessary because the magnitude of $\mathbf{QK}^\mathsf{T}$ grows with the embedding dimension $d_k$, which can result in gradient vanishing issues after the softmax operation.
Formally, assume that all the elements in $\bf Q$ and $\bf K$ are independent random variables with a  mean of 0 and a variance of 1, then each element in their product ${\bf QK}^{\rm T}$ has mean 0 and variance $d$.
By multiplying the product with a factor of $1/{\sqrt{d}}$, the variance of the product is scaled to 1.

While DSSA does not employ the softmax function, the surrogate gradient also suffers gradient vanishes without scaling.
However, directly using the scaling factor of VSA is not feasible. Due to the spike-based nature of the input in DST and the attention map in DSSA, we cannot assume that they possess a mean of 0 and a variance of 1.
Therefore, the scaling factor values in DSSA should be different from those used in VSA. In the following theorem, we present the mean and variance of DST.
\begin{theorem}[Mean and variance of DST]
Given spike input ${\bf X}\in\{0,1\}^{T\times p\times m}$, ${\bf Y}\in\{0,1\}^{T\times m\times q}$ and linear transformation $f(\cdot)$ with weight matrix ${\bf W}\in \mathbb{R}^{q\times q}$, ${\bf I}\in \mathbb{R}^{T\times p\times q}$ is the output of DST, ${\bf I}={\rm DST}({\bf X},{\bf Y};f(\cdot))$. Assume that all elements in $\bf X$ and $\bf Y$ are independent random variables, $x_{i_x,j_x}[t]$ in $\bf X$ subject to Bernoulli distribution $x_{i_x,j_x}[t]\sim B(f_x)$, and the output of linear transformation $f({\bf Y})$ has mean 0 and variance 1, we have ${\rm E}(I_{i_I,j_I}[t])=0$, ${\rm Var}(I_{i_I,j_I}[t]) = f_x m$.
Similarly, for ${\bf I}={\rm DST_T}({\bf X},{\bf Y};f(\cdot))$ and ${\bf Y}\in\{0,1\}^{T\times q\times m}$, ${\bf W}\in \mathbb{R}^{m\times m}$, we also have ${\rm E}(I_{i_I,j_I}[t])=0$, ${\rm Var}(I_{i_I,j_I}[t]) = f_x m$.
\label{theorem:mean and variance}
\end{theorem}
The proof of Theorem~\ref{theorem:mean and variance} can be found in the supplementary.
Accroding to Theorem~\ref{theorem:mean and variance}, we have $c_1=1/\sqrt{f_{X}d}$, $c_2=1/\sqrt{f_{Attn}HW/p^2}$.
Here $f_{X}$ and $f_{Attn}$ are the average firing rate of $\bf X$ and spiking attention map, respectively.

\subsection{Spike-driven Characteristic of DSSA}
\label{sec:spike-driven}
The spike-driven characteristic is important for SNNs, i.e., the computation is sparsely triggered with spikes and requires only synaptic operations.
Previous work has made a great effort in achieving spike-driven Transformers~\cite{zhou2023spikingformer, yao2023spike}.
In this subsection, we delve into the spike-driven characteristic of DSSA and prove that DSSA is spike-driven.
We first give the formal definition of the spike-driven characteristic.
\begin{definition}
A spiking neural network is spike-driven if the input currents of all neurons satisfy the following form:
\begin{equation}
    I_{i}[t] = \sum_j w_{i,j}s_{j}[t] = \sum_{j,s_{j}[t]\ne 0}w_{i,j},
\end{equation}
where $I_{i}[t]$ is the input current of the $i$-th postsynaptic neuron at time step $t$, $s_{j}[t]\in \{0,1\}$ is the spike output of the $j$-th presynaptic neuron, $w_{i,j}$ is the weight of the synaptic connection from neuron $j$ to neuron $i$.
\label{definition:spike-driven}
\end{definition}
This definition reveals the nature of the spike-driven characteristic, i.e., the accumulation of input current is sparsely triggered by spikes emitted from presynaptic neurons.
It is evident that commonly used linear layers and convolution layers satisfy this definition.

The DSSA has only two spiking neuron layers, including the spiking attention map layer and the output layer. Both of these layers receive input currents derived from DST.
Thus, we only need to validate that the DST satisfies the form in Definition~\ref{definition:spike-driven}.
We first validate the $\rm DST_T$ in Eq.~\eqref{eq:DSTT}
\begin{align}
{\bf I} &= {\bf X}{\bf W}^{\rm T}{\bf Y}^{\rm T},\\
I_{i,j}[t] &= \sum_{k=1}^m\sum_{l=1}^m x_{i,k}[t]w_{l,k}y_{j,l}[t]
= \sum_{\underset{(x_{i,k}[t]\land y_{j,l}[t])\ne0}{k,l}} w_{l,k}.
\end{align}
Slightly different from Definition~\ref{definition:spike-driven}, the spike input here is the logical AND of dual spikes. This can be viewed as a synaptic operation requiring dual spikes to trigger, which is the reason it is called dual spike transformation.
Similar to the $\rm DST_T$, the $\rm DST$ in Eq.~\eqref{eq:DST} can also be formulated as:
\begin{align}
{\bf I} &= {\bf X}{\bf Y}{\bf W},\\
I_{i,j}[t] &= \sum_{k=1}^m\sum_{l=1}^q x_{i,k}[t]y_{k,l}[t]w_{l,j}
= \sum_{\underset{(x_{i,k}[t]\land y_{k,l}[t])\ne0}{k,l}} w_{l,j}.
\end{align}
Thus, DSSA is spike-driven.

\begin{table*}[!t]
\centering
\caption{Architectures of SpikingResformer series. The output size corresponds to the input size of 224$\times$224. $D_i$ and $H_i$ are the embedding dimension and number of heads of MHDSSA in stage $i$, respectively. $p_i$ denotes that the MHDSSA in stage $i$ uses $p_i\times p_i$ convolution in DST. $R_i$ and $G_i$ denote the expansion ratio and embedding dimension per group of GWSFFN in stage $i$, respectively.}
\resizebox{\linewidth}{!}
{

\begin{tabular}{ccccccc}
\toprule
Stage&Output Size&Layer Name&SpikingResformer-Ti&SpikingResformer-S&SpikingResformer-M&SpikingResformer-L\\
\midrule
Stem&56$\times$56&Stem&\multicolumn{4}{c}{Conv 7$\times$7, stride 2, Maxpooling 3$\times$3, stride 2}\\
\midrule
\multirow{3}{*}{Stage 1}&\multirow{3}{*}{56$\times$56}&\multirow{2}{*}{MHDSSA}&
\multirow{3}{*}{$\begin{bmatrix}
    D_1=64\\
    H_1=1,p_1=4\\
    R_1=4,G_1=64
\end{bmatrix}\times1$}&
\multirow{3}{*}{$\begin{bmatrix}
    D_1=64\\
    H_1=1,p_1=4\\
    R_1=4,G_1=64
\end{bmatrix}\times1$}&
\multirow{3}{*}{$\begin{bmatrix}
    D_1=64\\
    H_1=1,p_1=4\\
    R_1=4,G_1=64
\end{bmatrix}\times1$}&
\multirow{3}{*}{$\begin{bmatrix}
    D_1=128\\
    H_1=1,p_1=4\\
    R_1=4,G_1=64
\end{bmatrix}\times1$}\\
&&&&&\\
&&GWSFFN&&&\\

\midrule
\multirow{4}{*}{Stage 2}&\multirow{4}{*}{28$\times$28}&Downsample&Conv 3$\times$3, 192, stride 2&Conv 3$\times$3, 256, stride 2&Conv 3$\times$3, 384, stride 2&Conv 3$\times$3, 512, stride 2\\
\cmidrule{3-7}
&&\multirow{2}{*}{MHDSSA}&
\multirow{3}{*}{$\begin{bmatrix}
    D_2=192\\
    H_2=3,p_2=2\\
    R_2=4,G_2=64
\end{bmatrix}\times2$}&
\multirow{3}{*}{$\begin{bmatrix}
    D_2=256\\
    H_2=4,p_2=2\\
    R_2=4,G_2=64
\end{bmatrix}\times2$}&
\multirow{3}{*}{$\begin{bmatrix}
    D_2=384\\
    H_2=6,p_2=2\\
    R_2=4,G_2=64
\end{bmatrix}\times2$}&
\multirow{3}{*}{$\begin{bmatrix}
    D_2=512\\
    H_2=8,p_2=2\\
    R_2=4,G_2=64
\end{bmatrix}\times2$}\\
&&&&&\\
&&GWSFFN&&&\\

\midrule
\multirow{4}{*}{Stage 3}&\multirow{4}{*}{14$\times$14}&Downsample&Conv 3$\times$3, 384, stride 2&Conv 3$\times$3, 512, stride 2&Conv 3$\times$3, 768, stride 2&Conv 3$\times$3, 1024, stride 2\\
\cmidrule{3-7}
&&\multirow{2}{*}{MHDSSA}&
\multirow{3}{*}{$\begin{bmatrix}
    D_3=384\\
    H_3=6,p_3=1\\
    R_3=4,G_3=64
\end{bmatrix}\times3$}&
\multirow{3}{*}{$\begin{bmatrix}
    D_3=512\\
    H_3=8,p_3=1\\
    R_3=4,G_3=64
\end{bmatrix}\times3$}&
\multirow{3}{*}{$\begin{bmatrix}
    D_3=768\\
    H_3=12,p_3=1\\
    R_3=4,G_3=64
\end{bmatrix}\times3$}&
\multirow{3}{*}{$\begin{bmatrix}
    D_3=1024\\
    H_3=16,p_3=1\\
    R_3=4,G_3=64
\end{bmatrix}\times3$}\\
&&&&&\\
&&GWSFFN&&&\\

\midrule
Classifier&1$\times$1&Linear&\multicolumn{4}{c}{1000-FC}\\

\bottomrule 
\end{tabular}
}

\label{tab:arch}
\end{table*}

\section{SpikingResformer}
In this section, we first introduce the overall architecture of the proposed SpikingResformer and compare it with existing SEW ResNet~\cite{fang2021deep} and Spikformer~\cite{zhou2023spikformer}.
Then we detail the design of the spiking resformer block.

\subsection{Overall Architecture}

The overall pipeline of SEW ResNet~\cite{fang2021deep}, Spikformer~\cite{zhou2023spikformer}, and the proposed SpikingResformer are shown in Fig.~\ref{fig:overview}.
As shown in Fig.~\ref{fig:overview}, Spikformer employs a Spiking Patch Splitting (SPS) module to project an image to $d$ dimensional feature with reduced spatial size.
However, the local information can only be poorly modeled by this shallow spiking convolutional module and it only generates single-scale feature maps.
On the other hand, SEW ResNet is much deeper than the SPS in Spikformer and have a greater capability in extracting multi-scale feature with the multi-stage architecture, but lacks the global self-attention mechanism which helps extract global information.
To overcome the limitations of these architectures while exploiting their advantages, we propose SpikingResformer, which combines the ResNet-based architecture and the spiking self-attention mechanism.

The details of the overall architecture of SpikingResformer series are listed in Tab.~\ref{tab:arch}.
Similar to the ResNet-based SNNs~\cite{fang2021deep,hu2021spiking}, our model starts with a stem architecture consisting of a 7$\times$7 convolution and a 3$\times$3 max pooling to pre-extract localized features and employs a multi-stage backbone to generates multi-scale feature maps.
In each stage, multiple spiking Resformer blocks are stacked sequentially.
Each spiking Resformer block consists of two modules, named Multi-Head Dual Spike Self-Attention (MHDSSA) block and Group-Wise Spiking Feed-Forward Network (GWSFFN).
A downsample layer is applied before each stage (except the first stage) to reduce the size of the feature maps and project them to a higher dimension (2$\times$ downsampling of resolution with 2$\times$ enlargement of dimension).
Finally, the model ends with a global average pooling layer and a classification layer.

\subsection{Spiking Resformer Block}
As illustrated in Fig.~\ref{fig:overview}, the spiking Resformer block consists of a Multi-Head Dual Spike Self-Attention (MHDSSA) module and a Group-Wise Spiking Feed-Forward Network (GWSFFN).
We first introduce the two modules and then derive the form of the spiking Resformer block.

\noindent
\textbf{Multi-Head Convolutional Spiking Self-Attention.}
In Sec.~\ref{sec:DSSA}, we propose the single-head form of DSSA.
It can be easily extended to the multi-head DSSA (MHDSSA) following a similar approach to the vanilla Transformer~\cite{vaswani2017attention}.
In MHDSSA, we first split the results of linear transformation in DST into $h$ heads, then perform DSSA on each head and concatenate them together.
Finally, we use point-wise convolution to project the concatenated features thus fusing the features in different heads.
The MHDSSA can be formulated as follows:
\begin{equation}
{\rm MHDSSA}({\bf X})={\rm BN}({\rm Conv_1}([\mathrm{DSSA}_i({\rm SN}({\bf X}))]_{i=1}^h)),
\end{equation}
where $[\dots]$ denotes the concatenate operation, $\rm Conv_1$ denotes the point-wise convolution.

\noindent
\textbf{Group-Wise Spiking Feed-Forward Network.}
The spiking feed-forward network (SFFN) proposed in previous spiking vision transformers is composed of two linear layers with batch normalization and spiking neuron activation~\cite{zhou2023spikformer,yao2023spike}.
Moreover, the expansion ratio is usually set to 4, i.e., the first layer raises the dimension by a factor of 4 while the second layer reduces to the original dimension.

Based on SFFN, we insert a 3$\times$3 convolution layer with the residual connection between two linear layers to enable SFFN to extract local features. Since the dimension of the hidden features between the two linear layers is expanded by a factor of 4 compared to the input, in order to reduce the number of parameters and computational overhead, we use group-wise convolution and set every 64 channels as 1 group.
We also employ the spike-driven design in~\cite{zhou2023spikingformer,yao2023spike}.
The group-wise spiking feed-forward network (GWSFFN) can be formulated as follows:
\begin{align}
{\rm FFL}_i({\bf X})&={\rm BN}({\rm Conv}_1({\rm SN}({\bf X}))),\;i=1,2,\\
{\rm GWL}({\bf X})&={\rm BN}({\rm GWConv}({\rm SN}({\bf X})))+{\bf X},\\
{\rm GWSFFN}({\bf X})&={\rm FFL}_2({\rm GWL}({\rm FFL}_1({\bf X}))).
\end{align}
Here ${\rm FFL}_i(\cdot),i=1,2$ denote the feed-forward layers, ${\rm Conv}_1(\cdot)$ is point-wise convolution (1$\times$1 convolution), which equal to the linear transformation. ${\rm GWL(\cdot)}$ denotes group-wise convolution layer, ${\rm GWConv(\cdot)}$ denotes the group-wise convolution.

\noindent
\textbf{Spiking Resformer Block.}
With the MHDSSA module and GWSFFN above, the spiking resformer block can be formulated as:
\begin{align}
{\bf Y}_i &= {\rm MHDSSA}({\bf X}_i) + {\bf X}_i,\\
{\bf X}_{i+1} &= {\rm GWSFFN}({\bf Y}_i) + {\bf Y}_i.
\end{align}
where ${\bf Y}_i$ denotes the output features of MHDSSA module in the $i$-th spiking resformer block.

\begin{table*}[!ht]
\centering
\caption{Evaluation on ImageNet. SOPs denotes the average synaptic operations of an image inference on ImageNet validation data. Energy is the estimation of energy consumption same as~\cite{zhou2023spikformer,yao2023spike}. The default input resolution for training and inference is 224$\times$224. $\dagger$~means the input is enlarged to 288$\times$288 in inference. -~means the data is not provided in the original paper.}
\resizebox{\linewidth}{!}
{
\begin{tabular}{cccccccc}
\toprule
\multirow{2}{*}{\bf Method} & \multirow{2}{*}{\bf Type} & \multirow{2}{*}{\bf Architecture} & \multirow{2}{*}{\bf T} & {\bf Param} & {\bf SOPs} & {\bf Energy} & {\bf Top-1 Acc.} \\
&&&&(M)&(G)&(mJ)&(\%)\\
\midrule

\multirow{2}{*}{Spiking ResNet~\cite{hu2021spiking}}&\multirow{2}{*}{ANN-to-SNN}
&ResNet-34&350&21.79&65.28&59.30&71.61\\
&&ResNet-50&350&25.56&78.29&70.93&72.75\\
\cmidrule{3-8}
STBP-tdBN~\cite{zheng2021going}&Direct Training&Spiking ResNet-34&6&21.79&6.50&6.39&63.72\\
\cmidrule{3-8}
\multirow{2}{*}{TET~\cite{deng2022temporal}}&\multirow{2}{*}{Direct Training}&Spiking ResNet-34&6&21.79&-&-&64.79\\
&&SEW ResNet-34&4&21.79&-&-&68.00\\
\cmidrule{3-8}
\multirow{4}{*}{SEW ResNet~\cite{fang2021deep}}
&\multirow{4}{*}{Direct Training}
&SEW ResNet-34&4&21.79&3.88&4.04&67.04\\
&&SEW ResNet-50&4&25.56&4.83&4.89&67.78\\
&&SEW ResNet-101&4&44.55&9.30&8.91&68.76\\
&&SEW ResNet-152&4&60.19&13.72&12.89&69.26\\

\midrule

\multirow{3}{*}{Spikformer~\cite{zhou2023spikformer}}
&\multirow{3}{*}{Direct Training}
&Spikformer-8-384&4&16.81&6.82&7.73&70.24\\
&&Spikformer-8-512&4&29.68&11.09&11.58&73.38\\
&&Spikformer-8-768&4&66.34&22.09&21.48&74.81\\
\cmidrule{3-8}
\multirow{3}{*}{Spikingformer~\cite{zhou2023spikingformer}}
&\multirow{3}{*}{Direct Training}
&Spikingformer-8-384&4&16.81&-&4.69&72.45\\
&&Spikingformer-8-512&4&29.68&-&7.46&74.79\\
&&Spikingformer-8-768&4&66.34&-&13.68&75.85\\
\cmidrule{3-8}
\multirow{4}{*}{Spike-driven Transformer~\cite{yao2023spike}}
&\multirow{4}{*}{Direct Training}
&Spike-driven Transformer-8-384&4&16.81&-&3.90&72.28\\
&&Spike-driven Transformer-8-512&4&29.68&-&4.50&74.57\\
&&Spike-driven Transformer-8-768&4&66.34&-/-&6.09/-&76.32/77.07${}^\dagger$\\
\cmidrule{3-8}
\multirow{4}{*}{\bf SpikingResformer (Ours)}
&\multirow{4}{*}{Direct Training}
&SpikingResformer-Ti&4&{\bf 11.14}&{\bf 2.73}/{\bf 4.71}${}^\dagger$&{\bf 2.46}/{\bf 4.24}${}^\dagger$&{\bf 74.34}/{\bf 75.57}${}^\dagger$\\
&&SpikingResformer-S&4&{\bf 17.76}&{\bf 3.74}/{\bf 6.40}${}^\dagger$&{\bf 3.37}/{\bf 5.76}${}^\dagger$&{\bf 75.95}/{\bf 76.90}${}^\dagger$\\
&&SpikingResformer-M&4&{\bf 35.52}&{\bf 6.07}/{\bf 10.24}${}^\dagger$&{\bf 5.46}/{\bf 9.22}${}^\dagger$&{\bf 77.24}/{\bf 78.06}${}^\dagger$\\
&&SpikingResformer-L&4&{\bf 60.38}&{\bf 9.74}/{\bf 16.40}${}^\dagger$&{\bf 8.76}/{\bf 14.76}${}^\dagger$&{\bf 78.77}/{\bf 79.40}${}^\dagger$\\

\bottomrule
\end{tabular}
}
\label{tab:ImageNet}
\end{table*}

\section{Experiments}

In this section, we first evaluate the performance and energy efficiency of SpikingResformer on the ImageNet classification task. Then, we perform ablation experiments on key components in SpikingResformer. Finally, we evaluate the transfer learning ability of SpikingResformer.

\subsection{ImageNet Classification}
\label{sec:ImageNet}

ImageNet~\cite{deng2009imagenet} is one of the most typical static image datasets widely used for image classification.
For a fair comparison, we generally follow the data augmentation strategy and training setup in~\cite{yao2023spike}.
More details of the experimental setup can be found in the supplementary.

\noindent
\textbf{Results.}
Our experimental results are listed in Tab.~\ref{tab:ImageNet}. For comparison, we also list the results of existing spiking convolutional networks and spiking Vision Transformers. As shown in Tab.~\ref{tab:ImageNet}, SpikingResformer achieves higher accuracy, fewer parameters, and lower energy consumption at the same time compared to existing methods. For instance, SpikingResformer-Ti achieves 74.34\% accuracy with only 11.14M parameters and 2.73G SOPs (2.46mJ), outperforming Spike-driven Transformer-8-384 by 2.06\%, saving 5.67M parameters and 1.44mJ energy.
SpikingResformer-M achieves 77.24\% accuracy with 35.52M parameters and 6.07G SOPs (5.46mJ), outperforming Spike-driven Transformer-8-768 by 0.92\% while saving 30.82M parameters.
Particularly, the SpikingResformer-L achieves up to 79.40\% accuracy when the input size is enlarged to 288$\times$288, which is the state-of-the-art result in SNN field.

\begin{table}[!t]
\centering
\caption{Ablation Study on ImageNet100 dataset. The number of parameters for all variants is comparable to SpikingResformer-S.}
\resizebox{\linewidth}{!}
{
\begin{tabular}{lccc}
\toprule
\bf Model & {\bf SOPs}~(G) & {\bf Energy}~(mJ) & {\bf Acc.}~(\%) \\
\midrule
SpikingResformer-S&2.43&2.18&88.06\\
w/o multi-stage architecture&1.84&1.66&85.32\\
w/o group-wise convolution&2.37&2.13&84.64\\
w/o DSSA&\multicolumn{3}{c}{not converge}\\
w/o $p\times p$ convolution&4.34&3.91&86.14\\
w/o scaling&\multicolumn{3}{c}{not converge}\\

\bottomrule
\end{tabular}
}
\label{tab:ablation}
\end{table}

\begin{table*}[!t]
\centering
\caption{\centering{Transfer learning results on CIFAR10, CIFAR100, CIFAR10-DVS, DVSGesture datasets.}}
\resizebox{0.8\linewidth}{!}
{
\begin{tabular}{cccccccccc}
\toprule
\multirow{2}{*}{\bf Method} & \multirow{2}{*}{\bf Type} & \multicolumn{2}{c}{CIFAR10} & \multicolumn{2}{c}{CIFAR100} & \multicolumn{2}{c}{CIFAR10-DVS} & \multicolumn{2}{c}{DVSGesture}\\
\cmidrule(lr){3-4}\cmidrule(lr){5-6}\cmidrule(lr){7-8}\cmidrule(lr){9-10}
&&\bf T&\bf Acc.\ &\bf T&\bf Acc.\ &\bf T&\bf Acc.\ &\bf T&\bf Acc.\ \\
\midrule
STBP-tdBN~\cite{zheng2021going}&Direct Training&6&93.16&-&-&10&67.8&40&96.87\\
PLIF~\cite{fang2021incorporating}&Direct Training&8&93.50&-&-&20&74.8&20&97.57\\
Dspike~\cite{li2021differentiable}&Direct Training&6&94.25&6&74.24&10&75.4&-&-\\
\midrule
Spikformer~\cite{zhou2023spikformer}&Direct Training&4&95.19&4&77.86&16&80.6&16&97.9\\
Spikingformer~\cite{zhou2023spikingformer}&Direct Training&4&95.61&4&79.09&16&81.3&16&98.3\\
Spike-driven Transformer~\cite{yao2023spike}&Direct Training&4&95.6&4&78.4&16&80.0&16&\bf 99.3\\
\midrule
Spikformer~\cite{zhou2023spikformer}&Transfer Learning&4&97.03&4&83.83 &-&-&-&- \\
{\bf SpikingResformer (Ours)}&Transfer Learning&4&\bf 97.40&4&\bf 85.98&10&\bf 84.8&10&93.4\\

\bottomrule
\end{tabular}
}
\label{tab:transfer learning}
\end{table*}

\subsection{Ablation Study}

We perform ablation experiments on key components in SpikingResformer, including the multi-stage architecture, the group-wise convolution layer in GWSFFN, and our proposed spiking self-attention mechanism.
The ablation experiments are conducted on the ImageNet100 dataset, which is a subset of the ImageNet dataset and consists of 100 categories from the ImageNet dataset.
The experimental setup basically follows the one in Sec.~\ref{sec:ImageNet}. Detailed settings are listed in the supplementary.

\noindent
\textbf{Multi-Stage Architecture.}
To verify the effectiveness of the multi-stage architecture, we replace it with the Spikingformer-based architecture, while the structure of the spiking Resformer block remains unchanged.
We adjust the embedding dimensions to make the model parameters comparable to SpikingResformer-S.
As shown in Tab.~\ref{tab:ablation}, SpikingResformer-S outperforms the variant without multi-stage architecture by 2.74\%, demonstrating the effectiveness of multi-stage architecture.

\noindent
\textbf{Group-Wise Convolution Layer.}
In comparison to the point-wise SFFN, GWSFFN employs a 3$\times$3 group-wise convolution layer between the two linear layers.
We evaluate its effect by removing the group-wise convolution layer and increasing the dimension to keep the number of parameters constant.
As demonstrated in Tab.~\ref{tab:ablation}, SpikingResformer-S achieves 3.42\% higher accuracy in contrast to the variant without the group-wise convolution layer.
This difference highlights the benefits of the group-wise convolution layer in GWSFFN.

\noindent
\textbf{Dual Spike Self-Attention.}
To validate the effectiveness of DSSA, we first replace DSSA with existing spiking self-attention mechanisms, including Spiking Self-Attention (SSA) in Spikformer and Spike-Driven Self-Attention (SDSA) in Spike-driven Transformer.
However, neither SSA nor SDSA converges.
We believe this is because SSA and SDSA do not fit the multi-scale input.
To further validate the key factors that help DSSA fit for the multi-scale input, we conduct two more sets of experiments.
One set replaces all $p\times p$ convolutions with $1\times 1$ convolutions in DST to validate the effect of reducing spatial size.
The other removes the scaling factor or replaces our proposed scaling factors with $1/\sqrt{d}$ to validate the effect of scaling.
As a result, the first group converges but only achieves 86.14\% accuracy with a higher energy consumption of 3.91mJ, while the second set does not converge.
This shows that the scaling factor is the key to convergence.
We believe that both SSA and SDSA lack proper scaling methods, thus do not apply to multi-stage architecture.
SSA employs the same scaling factor as vanilla self-attention, i.e., $1/\sqrt{d}$. However, this does not apply to spike matrix multiplication, since spikes do not have mean 0 and variance 1.
SDSA has no scaling method, which may be feasible when the feature map is small but not for multi-scale inputs.

\subsection{Transfer Learning}
High transfer learning ability is a key advantage of Vision Transformer.
We evaluate the transfer learning ability of the proposed SpikingResformer on static dataset CIFAR10 and CIFAR100~\cite{krizhevsky2009learning} and neuromorphic dataset CIFAR10-DVS~\cite{li2017cifar10} and DVSGesture~\cite{amir2017low} by fine-tuning the models pre-trained on ImageNet.
Among the existing spiking Vision Transformers, only Spikformer provides transfer learning results, and only on the static image datasets CIFAR10 and CIFAR100.
Thus, we also compared our results with the direct training results for a comprehensive comparison.
Tab.~\ref{tab:transfer learning} lists the highest accuracy achieved by these methods. The experimental setup and a more detailed comparison can be found in the supplementary.

\noindent
\textbf{Static Image Datasets.}
As shown in Tab.~\ref{tab:transfer learning}, SpikingResformer achieves 97.40\% accuracy on CIFAR10 dataset and 85.98\% accuracy on CIFAR100 dataset, which is the state-of-the-art result, outperforming the transfer learning results of Spikformer by 0.37\% on CIFAR10 and 2.15\% on CIFAR100.
Compared to direct training methods, SpikingResformer obtained from transfer learning has significantly better performance.
For example, SpikingResformer outperforms Spikingformer by 6.89\% on CIFAR100 dataset, demonstrating the advantage of transfer learning.

\noindent
\textbf{Neuromorphic Datasets.}
Neuromorphic datasets differ significantly from traditional static image datasets. The samples of the neuromorphic dataset consist of event streams instead of RGB images.
As a result, there is a large gap between the source and target domain for models pre-trained on static image datasets.
To bridge this gap, we stack the events over a period of time to form a frame. The RGB channels are replaced with positive events, negative events, and the sum of events channels, respectively.
As shown in Tab.~\ref{tab:transfer learning}, transfer learning results of SpikingResformer significantly outperform the direct training ones on CIFAR10-DVS.
SpikingResformer achieves 84.8\% accuracy on CIFAR10-DVS, outperforming Spikingformer by 3.5\%.
However, the transfer learning results on DVSGesture fail to achieve comparable performance to direct training.
SpikingResformer only achieves 93.4\% accuracy on DVSGesture, falling behind the state-of-the-art method Spike-driven Transformer by 5.9\%.
We believe that this is mainly due to the way CIFAR10-DVS is constructed differs from DVSGesture.
CIFAR10-DVS is converted from CIFAR10, which does not contain temporal information.
Thus, models pre-trained on static datasets can transfer to CIFAR10-DVS well.
However, DVSGesture is directly created from human gestures using a dynamic vision sensor, thus containing rich temporal information. As a result, models pre-trained on static datasets do not transfer well to DVSGesture.
We hope our exploration of transfer learning on neuromorphic datasets could pave the way for further transfer learning research of SNNs.

\section{Conclusion}

In this paper, we propose a novel spiking self-attention mechanism named Dual Spike Self-Attention (DSSA). It produces spiking self-attention via Dual Spike Transformation, which is fully spike-driven and compatible with SNNs.
We detail the scaling factors in DSSA enabling it to handle feature maps of arbitrary scales.
Based on DSSA, we propose SpikingResformer, which combines the ResNet-based multi-stage architecture with our proposed DSSA to achieve superior performance and energy efficiency with fewer parameters.
Extensive experiments demonstrate the effectiveness and superiority of the proposed SpikingResformer.

\section*{Acknowledgments}
This work was supported by the National Natural Science Foundation of China (62176003, 62088102) and by the Beijing Nova Program (20230484362).

{
    \small
    \bibliographystyle{ieeenat_fullname}
    \bibliography{main}
}

\clearpage
\setcounter{section}{0}
\setcounter{table}{0}
\setcounter{figure}{0}
\setcounter{equation}{0}
\maketitlesupplementary

\renewcommand{\thesection}{\Alph{section}}
\renewcommand{\thetable}{S\arabic{table}}
\renewcommand{\thefigure}{S\arabic{figure}}
\renewcommand{\theequation}{S\arabic{equation}}

\newtheorem{lemma}{Lemma}

\section{Proof of Theorem 1}

\begin{lemma}{\rm (}{\bf Expectation and variance of the product of independent random variables}{\rm )}
\label{lemma:1}
Given two independent random variables $a$ and $b$ with expectation and variance, we have
\begin{align}
{\rm E}(ab)&={\rm E}(a){\rm E}(b),\\
{\rm Var}(ab)&={\rm Var}(a){\rm Var}(b)+{\rm Var}(a){\rm E}(b)^2+{\rm Var}(b){\rm E}(a)^2.
\end{align}
\end{lemma}
\begin{proof}
The expectation of $ab$ can be formulated as:
\begin{equation}
{\rm E}(ab)={\rm E}(a){\rm E}(b)+{\rm Cov}(a,b).
\end{equation}
Since the random variables $a$ and $b$ are independent of each other, the covariance ${\rm Cov}(a,b)=0$. Thus, we have
\begin{equation}
{\rm E}(ab)={\rm E}(a){\rm E}(b)+0={\rm E}(a){\rm E}(b).
\end{equation}
Using the above conclusion and the definition of variance, we have
\begin{equation}
\begin{aligned}
&{\rm Var}(ab)={\rm E}((ab-{\rm E}(ab))^2)\\
&={\rm E}(a^2b^2)-{\rm E}(ab)^2\\
&={\rm E}(a^2){\rm E}(b^2)-{\rm E}(a)^2{\rm E}(b)^2\\
&=({\rm Var}(a)+{\rm E}(a)^2)({\rm Var}(b)+{\rm E}(b)^2)-{\rm E}(a)^2{\rm E}(b)^2\\
&={\rm Var}(a){\rm Var}(b)+{\rm Var}(a){\rm E}(b)^2+{\rm Var}(b){\rm E}(a)^2.
\end{aligned}
\end{equation}
\end{proof}

\begin{lemma}{\rm (}{\bf Expectation and variance of the sum of independent random variables}{\rm )}
\label{lemma:2}
Given independent random variables $a_1,a_2,\dots,a_n$ with expectation and variance, we have
\begin{align}
{\rm E}(\sum_{i=1}^n a_i)&=\sum_{i=1}^n {\rm E}(a_i),\\
{\rm Var}(\sum_{i=1}^n a_i)&=\sum_{i=1}^n {\rm Var}(a_i).
\end{align}
\end{lemma}
\begin{proof}
Considering first the case of two independent random variables $a_i$ and $a_j$ where $i\ne j$, the covariance ${\rm Cov}(a_i,a_j)=0$, we have

\begin{align}
{\rm E}(a_i+a_j)&={\rm E}(a_i)+{\rm E}(a_j),\\
{\rm Var}(a_i+a_j)&={\rm Var}(a_i)+{\rm Var}(a_j)+2{\rm Cov}(a_i,a_j)\notag \\&={\rm Var}(a_i)+{\rm Var}(a_j)
\end{align}
This can be simply generalized to the case of $n$ random variables as:
\begin{align}
{\rm E}(\sum_{i=1}^n a_i)&=\sum_{i=1}^n {\rm E}(a_i),\\
{\rm Var}(\sum_{i=1}^n a_i)&=\sum_{i=1}^n {\rm Var}(a_i)+\sum_{1\le i,j\le n,i\ne j} {\rm Cov}(a_i,a_j)\notag \\
&=\sum_{i=1}^n {\rm Var}(a_i).
\end{align}
\end{proof}

With Lemma~\ref{lemma:1} and Lemma~\ref{lemma:2}, we prove the Theorem 1 in the main text.
\begin{proof}
Let us first consider the case of ${\rm DST_T}({\bf X},{\bf Y};f(\cdot))$.
We denote the result of applying linear transformation $f$ on $\bf Y$ as ${\bf Z}=f({\bf Y})={\bf YW}$.
Thus, each element in the result of $\rm DST_T$ can be formulated as:
\begin{equation}
I_{i,j}[t]=\sum_{k=1}^m x_{i,k}[t]z_{j,k}[t]=\sum_{k=1}^m x_{i,k}[t]\left( \sum_{l=1}^m y_{j,l}[t]w_{l,k}[t]\right)
\end{equation}
Based on the assumption, each $z_{j,k}[t]$ has a mean of 0 and variance of 1, and each $x_{i,k}[t]$ subjects to Bernoulli distribution $B(f_x)$, thus we have ${\rm E}(x_{i,k}[t])=f_x$ and ${\rm Var}(x_{i,k}[t])=f_x(1-f_x)$. According to Lemma~\ref{lemma:1}, we have
\begin{align}
{\rm E}(x_{i,k}[t]z_{j,k}[t])=&0\cdot f_x=0,\label{eq:exp}\\
{\rm Var}(x_{i,k}[t]z_{j,k}[t])=&1\cdot f_x(1-f_x)+1\cdot f_x^2\notag \\&+f_x(1-f_x)\cdot0^2\notag \\
=&f_x.\label{eq:var}
\end{align}
In addition, each $z_{j,k}[t],k=1,\dots,q$ consists of a set of $y_{l,k}[t],l=1,\dots,q$ that do not overlap each other. Thus $z_{j,k}[t]$ can also be viewed as independent random variables. According to Lemma~\ref{lemma:2}, we have
\begin{align}
{\rm E}(I_{i,j}[t])&=\sum_{k=1}^m 0=0,\label{eq:exp final}\\
{\rm Var}(I_{i,j}[t])&=\sum_{k=1}^m f_x=f_x m.\label{eq:var final}
\end{align}
Similar to $\rm DST_T$, each element in the result of $\rm DST$ can be formulated as:
\begin{equation}
I_{i,j}[t]=\sum_{k=1}^m x_{i,k}[t]z_{k,j}[t]=\sum_{k=1}^m x_{i,k}[t]\left( \sum_{l=1}^q y_{k,l}[t]w_{l,j}[t]\right)
\end{equation}
And we also have ${\rm E}(x_{i,k}[t]z_{k,j}[t])=0$ and ${\rm Var}(x_{i,k}[t]z_{j,k}[t])=f_x$. Thus, the expectation and variance of $I_{i,j}[t]$ are also ${\rm E}(I_{i,j}[t])=0$ and ${\rm Var}(I_{i,j}[t])=f_x m$.
\end{proof}

\noindent
\textbf{Further Discussion.}
In Eq.~\eqref{eq:exp final} and Eq.~\eqref{eq:var final}, we treat $z_{j,k}[t]$ as independent random variables, since each $z_{j,k}[t],k=1,\dots,q$ consists of a set of $y_{l,k}[t],l=1,\dots,q$ that do not overlap each other.
It is true for the $p\times p$ convolution with stride $p$ used in this paper.
However, this property does not hold for all generalized linear transformations.
For example, the $3\times 3$ convolution with stride $1$ leads to input overlap.
For these operations, we need a stronger assumption that assuming $z_{j,k}[t]$ are independent random variables.

\section{Scaling Factors in Existing Spiking Self-Attention Mechanisms}
In the main text, we propose that existing spiking self-attention mechanisms lack reasonable scaling methods and design scaling factors for our DSSA.
In this section, we use a similar approach to design scaling factors for these existing methods and thereby analyze the limitations of these methods.

\noindent
\textbf{Scaling Factor in Spiking Self-Attention (SSA).}
First, we try to design the scaling factor for the Spiking Self-Attention (SSA) in Spikformer~\cite{zhou2023spikformer}. The SSA can be formulated as follows:
\begin{align}
    {\bf Q} &= {\rm SN}({\rm BN}({\bf XW}_Q)),\label{eq:SSA Q} \\
    {\bf K} &= {\rm SN}({\rm BN}({\bf XW}_K)), \label{eq:SSA K}\\
    {\bf V} &= {\rm SN}({\rm BN}({\bf XW}_V)), \label{eq:SSA V}\\
    \mathrm{SSA}({\bf Q},{\bf K},{\bf V}) &= \mathrm{SN}({\bf QK}^{\rm T}{\bf V} * c),
\end{align}
where ${\bf X}\in\mathbb{R}^{HW\times d}$ is the input, ${\bf W}_Q,{\bf W}_K,{\bf W}_V\in\mathbb{R}^{d\times d}$ are weight matrices, $H$ and $W$ are the height and width of input, respectively, $d$ is the embedding dimension, $c$ is the scaling factor.
We denote ${\bf I}={\bf QK}^{\rm T}{\bf V}$, each element in ${\bf I}$ can be formulated as:
\begin{equation}
I_{i,j}[t]=\sum_{r=1}^{d}\sum_{l=1}^{HW}q_{i,r}[t]k_{r,l}[t]v_{l,j}[t].
\end{equation}
Assume that all elements in $\bf Q$, $\bf K$, and $\bf V$ are independent random variables, $q_{i_q,j_q}[t]$ in $\bf Q$ subject to Bernoulli distribution $q_{i_q,j_q}[t]\sim B(f_Q)$, $k_{i_k,j_k}[t]$ in $\bf K$ subject to $B(f_K)$, $v_{i_v,j_v}[t]$ in $\bf V$ subject to $B(f_V)$, respectively, $f_Q$, $f_K$, and $f_V$ are the average firing rate of $\bf Q$, $\bf K$, and $\bf V$, respectively.
We have ${\rm E}(I_{i,j}[t])=HWdf_Qf_Kf_V$.

However, the form of variance is complex.
This is because the summation terms $q_{i,r}[t]k_{r,l}[t]v_{l,j}[t]$ are not independent thus introducing a lot of covariance. The variance can be formulated as:
\begin{equation}
\begin{aligned}
&{\rm Var}(I_{i,j}[t])\\
=&\sum_{r=1}^{d}\sum_{l=1}^{HW}\bigg({\rm Var}(q_{i,r}[t]k_{r,l}[t]v_{l,j}[t])\\&+\sum_{r^\prime\ne r}{\rm Cov}(q_{i,r}[t]k_{r,l}[t]v_{l,j}[t],q_{i,r^\prime}[t]k_{r^\prime,l}[t]v_{l,j}[t])\\&+\sum_{l^\prime\ne l}{\rm Cov}(q_{i,r}[t]k_{r,l}[t]v_{l,j}[t],q_{i,r}[t]k_{r,l^\prime}[t]v_{l^\prime,j}[t])\bigg)\\
=&HWd\bigg(f_Qf_Kf_V(1-f_Q)(1-f_K)(1-f_V)\\
&+f_Qf_Kf_V^2(1-f_Q)(1-f_K)\\
&+f_Qf_K^2f_V(1-f_Q)(1-f_V)\\
&+f_Q^2f_Kf_V(1-f_K)(1-f_V)\\
&+f_Qf_K^2f_V^2(1-f_Q)\\
&+f_Q^2f_Kf_V^2(1-f_K)\\
&+f_Q^2f_K^2f_V(1-f_V)\\
&+(d-1)(f_Q^2f_K^2f_V-f_Q^2f_K^2f_V^2)\\
&+(HW-1)(f_Qf_K^2f_V^2-f_Q^2f_K^2f_V^2)\bigg)\\
=&HWdf_Qf_Kf_V\bigg(1-(HW+d-1)f_Qf_Kf_V\\
&+(d-1)f_Qf_K+(HW-1)f_Kf_V\bigg).
\end{aligned}
\label{eq:var SSA}
\end{equation}
As shown in Eq.~\eqref{eq:var SSA}, this form is overly complex and lacks practicality. Thus it is difficult to design the scaling factor for SSA.

\noindent
\textbf{Scaling Factor in Spike-driven Self-Attention (SDSA).}
Next, we try to design the scaling factor for the Spike-driven Self-Attention (SDSA) in Spike-driven Transformer~\cite{yao2023spike}. The SDSA can be formulated as follows:
\begin{equation}
{\rm SDSA}({\bf Q},{\bf K},{\bf V})={\rm SN}({\rm SUM_c}({\bf Q}\otimes{\bf K}))\otimes{\bf V},
\label{eq:SDSA}
\end{equation}
where $\otimes$ denotes Hadamard product, $\rm SUM_c$ represents the sum of each column, $\bf Q$, $\bf K$, and $\bf V$ are the same as in Eq.~\eqref{eq:SSA Q} to Eq.\eqref{eq:SSA V}.
The original SDSA does not have a scaling factor. We believe that there should be a scaling factor before the spiking neuron layer and the Eq.~\eqref{eq:SDSA} should be reformulated as follows:
\begin{equation}
{\rm SDSA}({\bf Q},{\bf K},{\bf V})={\rm SN}({\rm SUM_c}({\bf Q}\otimes{\bf K})*c)\otimes{\bf V},
\end{equation}
where $c$ is the scaling factor.
We denote ${\bf I}={\rm SUM_c}({\bf Q}\otimes{\bf K})$, each element in ${\bf I}$ can be formulated as:
\begin{equation}
I_{j}[t]=\sum_{i=1}^{HW}q_{i,j}[t]k_{i,j}[t].
\end{equation}
Following the same assumption in SSA, we have
\begin{equation}
\begin{aligned}
{\rm Var}(I_j[t])&=\sum_{i=1}^{HW}{\rm Var}(q_{i,j}[t]k_{i,j}[t])\\
&=\sum_{i=1}^{HW}(f_Qf_K(1-f_Qf_K))\\
&=HWf_Qf_K(1-f_Qf_K).
\end{aligned}
\end{equation}
Thus, the scaling factor $c$ in SDSA should be $c=1/\sqrt{HWf_Qf_K(1-f_Qf_K)}$.

To validate the effectiveness of this scaling factor, we conduct two sets of further ablation experiments.
One set introduces our proposed scaling factor to the Spike-driven Transformer.
The other replaces the DSSA in SpikingResformer with the SDSA with our proposed scaling factor. The $p\times p$ convolutions and the GWSFFN remain unchanged.
Experimental results are listed in Tab.~\ref{tab:ablation SDSA}.
As shown in Tab.~\ref{tab:ablation SDSA}, the scaling factor successfully solves the non-converge problem, demonstrating the effectiveness of our proposed scaling factor and its necessity for multi-scale feature map inputs.
Moreover, the scaling factor also improves the performance of SDSA with single-scale feature map inputs.

\begin{table}[!t]
\centering
\caption{Further ablation study on ImageNet100 dataset. The number of parameters for all variants is comparable to SpikingResformer-S.}
\resizebox{\linewidth}{!}
{
\begin{tabular}{lc}
\toprule
\bf Model & {\bf Acc.}~(\%) \\
\midrule
SpikingResformer-S&88.06\\
Spike-driven Transformer-8-384 (w/o scaling)&83.06\\
Spike-driven Transformer-8-384 (with scaling)&83.96\\
SpikingResformer-S with SDSA (w/o scaling)& not-converge\\
SpikingResformer-S with SDSA (with scaling)&87.74\\

\bottomrule
\end{tabular}
}
\label{tab:ablation SDSA}
\end{table}

\begin{figure}[!t]
    \centering
    \includegraphics[width=0.9\linewidth]{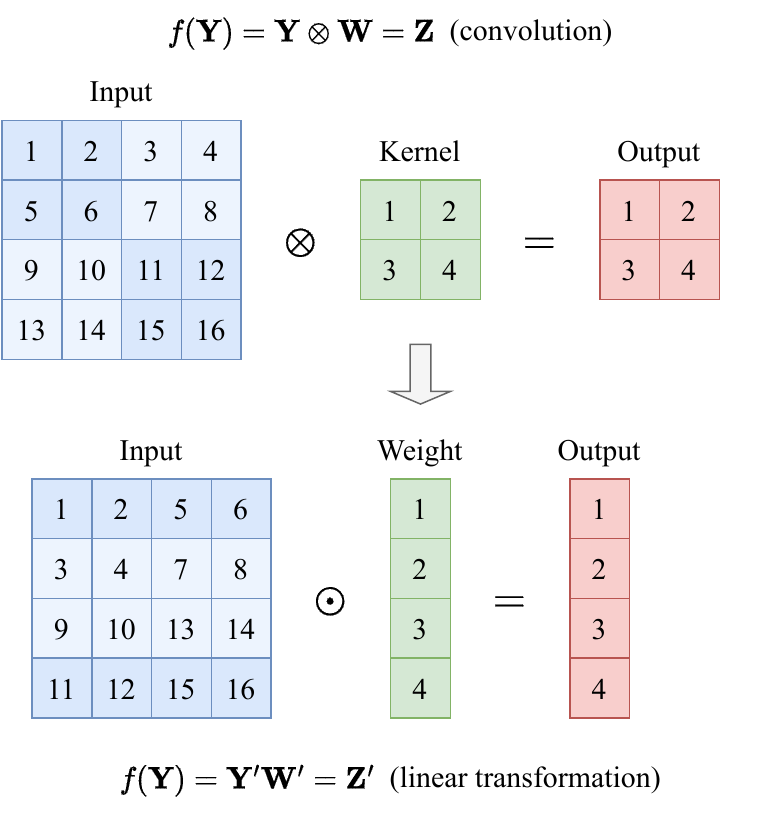}
    \caption{Diagram of the equivalence of convolution to linear transformation. {\bf Top}: ${\rm Conv_p}(\cdot)$ on a 4$\times$4 input where $p=2$; {\bf Bottom}: Its equivalent linear transformation.}
    \label{fig:convolution}
\end{figure}

\section{Equivalence of Convolution to Linear Transformation}

In this section, we discuss the equivalence of convolution to linear transformation.
For ease of understanding, we first visualize a simple example, and then give a formal description of the equivalence.
Since no dynamics in the temporal domain are involved here, we omit the time dimension.

Fig.~\ref{fig:convolution} shows how a 2$\times$2 convolution with a stride of 2 on a 4$\times$4 input is equivalent to a linear transformation.
In order to convert the convolution to its equivalent linear transformation, we first reshape the $h\times w$ convolution kernel $\bf W$ to a $hw\times 1$ weight matrix ${\bf W}^\prime$, where $h$ and $w$ are the height and width of the convolution kernel, respectively. Here $h=w=2$.
Then, we rewrite the $H_{in}\times W_{in}$ input $\bf Y$ to $H_{out}W_{out}\times hw$ input ${\bf Y}^\prime$, where $H_{in}$ and $W_{in}$ are the height and width of the input, $H_{out}$ and $W_{out}$ are the height and width of the output, respectively. Here $H_{in}=W_{in}=4$, $H_{out}=W_{out}=2$.
Each row in ${\bf Y}^\prime$ is a patch of input corresponding to an element in the output.
By the above process, we convert the convolution to its equivalent linear transformation.

We give a formal description of the equivalence of convolution to linear transformation.
Given an input ${\bf Y}\in\{0,1\}^{H_{in}\times W_{in}\times C_{in}}$ and a convolution kernel ${\bf W}\in\mathbb{R}^{h\times w\times C_{out}\times C_{in}}$, the convolution with stride $p$ can be formulated as follows:
\begin{align}
&{\bf Z}={\bf Y}\otimes{\bf W},\\
&z_{i,j,c_{out}}\notag \\&=\sum_{k=1}^h\sum_{l=1}^w\sum_{c_{in}=1}^{C_{in}}\left(y_{(i-1)*p+k,(j-1)*p+l,c_{in}}\cdot w_{k,l,c_{out},c_{in}}\right).
\end{align}
Let $g_Y$ be a mapping from $\{0,1\}^{H_{in}\times W_{in}\times C_{in}}$ to $\{0,1\}^{H_{out}W_{out}\times hwC_{in}}$ and $g_W$ be a mapping from $\mathbb{R}^{h\times w\times C_{out}\times C_{in}}$ to $\mathbb{R}^{hwC_{in}\times C_{out}}$, where
\begin{align}
&{\bf Y}^\prime=g_Y({\bf Y}),\notag \\
s.t.\;&y^\prime_{i*W_{out}+j,c_{in}*hw+k*w+l}=y_{(i-1)*p+k,(j-1)*p+l,c_{in}},\\
&{\bf W}^\prime=g_W({\bf W}),\notag \\
s.t.\;&w^\prime_{c_{in}*hw+k*w+l,c_{out}}=w_{k,l,c_{out},c_{in}}.
\end{align}
The linear transformation of $g_Y({\bf Y})$ with weight matrix $g_W({\bf W})$ can be formulated as:
\begin{align}
&{\bf Z}^\prime={\bf Y}^\prime{\bf W}^\prime=g_Y({\bf Y})g_W({\bf W}),\\
&z^\prime_{i*W_{out}+j,c_{out}}\notag \\
=&\sum_{k=1}^h\sum_{l=1}^w\sum_{c_{in}=1}^{C_{in}}\notag \\
&\left(y^\prime_{i*W_{out}+j,c_{in}*hw+k*w+l}\cdot w^\prime_{c_{in}*hw+k*w+l,c_{out}}\right) \notag\\
=&\sum_{k=1}^h\sum_{l=1}^w\sum_{c_{in}=1}^{C_{in}}\left(y_{(i-1)*p+k,(j-1)*p+l,c_{in}}\cdot w_{k,l,c_{out},c_{in}}\right)\notag \\
=&z_{i,j,c_{out}}.
\end{align}
Thus, the convolution is equivalent to linear transformation.

\section{Self-Attention in DSSA}
Dual Spike Self-Attention (DSSA) has no explicit Query, Key, and Value, which makes it quite different from the form of the Vanilla Self-Attention (VSA).
In this section, we further discuss how DSSA achieves self-attention.

\noindent
\textbf{Recall.} The DSSA can be formulated as follows:
\begin{align}
{\rm DSSA}({\bf X})&={\rm SN}({\rm DST}({\rm AttnMap}({\bf X}),{\bf X};f(\cdot))*c_2),\label{eq:recall attn} \\
{\rm AttnMap}({\bf X})&={\rm SN}({\rm DST_T}({\bf X}, {\bf X};f(\cdot))*c_1),\\
f({\bf X})&={\rm BN}({\rm Conv_p}({\bf X})).
\end{align}
And the Dual Spike Transformation (DST) can be formulated as follows:
\begin{align}
{\rm DST}({\bf X}, {\bf Y}; f(\cdot))&={\bf X}f({\bf Y})={\bf XYW},\\
{\rm DST_T}({\bf X}, {\bf Y}; f(\cdot))&={\bf X}f({\bf Y})^{\rm T}={\bf XW}^{\rm T}{\bf Y}^{\rm T}.
\end{align}

DSSA achieves self-attention by the two DSTs.
The first one, ${\rm DST_T}({\bf X}, {\bf X};f(\cdot))$, produces the attention map. It computes the multiplicative attention of a pixel in the input $\bf X$ and a $p\times p$ patch of feature transformed by the $p\times p$ convolution.
The output of ${\rm DST_T}({\bf X}, {\bf X};f(\cdot))$ is then scaled and fed to the spiking neuron as the input current to generate the spiking attention map.
The spiking attention map is a binary attention map consisting of spikes. Each spike $s_{i,j}$ in this spiking attention map signifies attention between the patch $i$ (pixel $i$) and patch $j$.
The second one, ${\rm DST}({\rm AttnMap}({\bf X}), {\bf X};f(\cdot))$,  produces the output feature.
For each pixel, it computes the sum of features of patches that have attention to this pixel to form the output features.
In this way, the first DST is similar to the product of ${\bf QK}^{\rm T}$ in the VSA, and the second DST is similar to the product of attention map and ${\bf V}$ in the VSA.

\section{Experiment Details}

\noindent
\textbf{ImageNet Classification.}
ImageNet~\cite{deng2009imagenet} is a vast collection of static images and one of the most commonly used datasets in computer vision tasks. It consists of around 1.2 million high-resolution images, categorized into 1,000 distinct classes. Each class includes approximately 1,000 images, representing a diverse range of objects and scenes, making it an effective reflection of real-world scenarios.

For ImageNet classification experiments, we generally follow the data augmentation strategy and training setup in~\cite{yao2023spike}.
We use the standard preprocessing, i.e., data normalization, randomly crop and resize the input to 224$\times$224 during traning, and set the input size to 224$\times$224 and 288$\times$288 for inferince.
We employ the standard data augmentation methods including random augmentation, mixup, cutmix, and label smoothing\footnote{Implemented by \href{https://github.com/huggingface/pytorch-image-models}{PyTorch Image Models}}, similar to~\cite{yao2023spike}.
We use the AdamW optimizer with a weight decay of 0.01.
The batch size varies from 256 (SpikingResformer-Ti) to 128 (SpikingResformer-L) depending on the model size.
We train the models for 320 epochs with a cosine-decay learning rate whose initial value varies from 0.001 (SpikingResformer-Ti) to 0.0005 (SpikingResformer-L).

Since the scaling factors in DSSA require the firing rate of input $f_X$ and attention map $f_{Attn}$, we use an exponential moving average with a momentum of 0.999 to count the average firing rate during training, and use the average firing rate counted during training in inference.
We used the same method to count the average firing rate in all subsequent experiments.

\noindent
\textbf{Ablation Study.}
All the ablation experiments are conducted on the ImageNet100 dataset.
It is a subset of the ImageNet dataset consisting of 100 categories from the original ImageNet dataset.
The experimental setup basically follows the ImageNet classification experiments.
The weight decay is increased to 0.05 since the ImageNet100 is smaller and easy to overfit.

\noindent
\textbf{Transfer Learning on Static Image Datasets.}
We first perform transfer learning experiments on static image datasets CIFAR10 and CIFAR100~\cite{krizhevsky2009learning}.
The CIFAR-10 dataset comprises 60,000 samples, divided into 10 categories with 6,000 samples in each category. Each group has 5,000 training samples and 1,000 testing samples. The images in the dataset are colored and have a resolution of 32$\times$32 pixels. On the other hand, the CIFAR-100 dataset is an extension of the CIFAR-10 dataset, designed to provide a more challenging and diverse benchmark for image recognition algorithms. It contains 100 classes for classification, encompassing a broader range of objects and concepts than the CIFAR-10 dataset's limited set of 10 classes.

We finetune the SpikingResformer-Ti and Spiking\-Resformer-S pretrained in ImageNet classification on these datasets.
We first replace the 1000-FC classifier layer with a randomly initialized 10-FC (CIFAR10) or 100-FC (CIFAR100) layer.
We finetune the model for 100 epochs with an initial learning rate of $1\times10^{-4}$ and cosine-decay to $1\times10^{-5}$.
The batch size is set to 128.
We employ data augmentation methods including random augmentation, mixup, and label smoothing.
We use the AdamW optimizer with a weight decay of 0.01.

\noindent
\textbf{Transfer Learning on Neuromorphic Datasets.}
We also perform transfer learning experiments on neuromorphic dataset CIFAR10-DVS~\cite{li2017cifar10} and DVSGesture~\cite{amir2017low}.
The CIFAR10-DVS dataset~\citep{li2017cifar10} is created by converting the static images in CIFAR10. This is done by moving the images and capturing the movement using a dynamic vision sensor. The CIFAR10-DVS dataset consists of 10,000 samples, with 1,000 samples per category. Each sample is an event stream with a spatial size of 128$\times$128.
It is worth noting that the CIFAR10-DVS dataset does not have predefined training and test sets. In our experiments, we select the first 900 samples of each category for training and the last 100 for testing.

The DVSGesture~\cite{amir2017low} dataset is created by directly capturing the human gestures using the DVS128 dynamic vision sensor.
It has 1,342 instances of 11 hand and arm gestures. These gestures were grouped in 122 trials, performed by 29 subjects under 3 different lighting conditions. The dataset includes hand waving, arm rotations, air guitar, etc.

We use the following preprocessing procedure. Firstly, we divide the event stream into ten slices, each of which contains an equal number of events. Next, for each slice, we stack the events into a single frame consisting of three channels. These channels represent positive events, negative events, and all events. Finally, we use this frame as the input for that particular time step. In this way, we use a time step of 10 for these datasets.

We use the data augmentation technique proposed in~\citep{li2022neuromorphic}. Other settings follow the experiments of transfer learning on static image datasets.

\begin{table}[!t]
\centering
\caption{Further comparison with ANN version SpikingResformer. $\dagger$~means the input is enlarged to 288$\times$288 in inference.}
\resizebox{\linewidth}{!}
{
\begin{tabular}{cccccc}
\toprule
\multirow{2}{*}{\bf Model} & \multirow{2}{*}{\bf T} & {\bf Param} & {\bf OPs} & {\bf Energy} & {\bf Top-1 Acc.} \\
&&(M)&(G)&(mJ)&(\%)\\
\midrule
Resformer-Ti&1&11.14&4.07&18.72&78.37\\
\midrule
SpikingResformer-Ti&4&{11.14}&{2.73}/{4.71}${}^\dagger$&{2.46}/{4.24}${}^\dagger$&{74.34}/{75.57}${}^\dagger$\\
SpikingResformer-S&4&{17.76}&{3.74}/{6.40}${}^\dagger$&{3.37}/{5.76}${}^\dagger$&{75.95}/{76.90}${}^\dagger$\\
SpikingResformer-M&4&{35.52}&{6.07}/{10.24}${}^\dagger$&{5.46}/{9.22}${}^\dagger$&{77.24}/{78.06}${}^\dagger$\\
SpikingResformer-L&4&{60.38}&{9.74}/{16.40}${}^\dagger$&{8.76}/{14.76}${}^\dagger$&{78.77}/{79.40}${}^\dagger$\\

\bottomrule
\end{tabular}
}
\label{tab:compare with ANN}
\end{table}

\section{Further Comparison with ANN Version}
We compare SpikingResformer with its ANN version, called Resformer in the following.
To construct the ANN version of SpikingResformer, we replace the spiking neurons in SpikingResformer with ReLU activation, and replace the neurons in the attention map with Softmax function.
For a fair comparison, other modules and the overall architecture remain unchanged.
In addition, the experimental setup of Resformer is the same as SpikingResformer.
As shown in Tab.~\ref{tab:compare with ANN}, the Resformer-Ti achieves higher accuracy (78.37\% vs.\ 74.34\%) but consumes significantly more energy (18.72mJ vs. 2.46mJ), even surpassing the energy consumption of SpikingResformer-L (8.76mJ, 78.77\% acc.), demonstrating the energy efficiency advantage of SNNs.

\begin{table*}[!t]
\centering
\caption{\centering{Detailed comparison on static datasets.}}
\footnotesize
{
\begin{threeparttable}
\begin{tabular}{ccccccc}
\toprule
\multirow{2}{*}{\bf Method} & \multirow{2}{*}{\bf Type} & \multirow{2}{*}{\bf Architecture } & \multirow{2}{*}{\#{\bf Param}~(M)} & \multirow{2}{*}{\bf T} & \multicolumn{2}{c}{{\bf Top-1 Acc.\ }(\%)}\\
\cmidrule{6-7}
&&&&&CIFAR10 & CIFAR100\\
\midrule
\multirow{3}{*}{STBP-tdBN~\cite{zheng2021going}}
&\multirow{3}{*}{Direct Training}
&\multirow{3}{*}{ResNet-19}
&\multirow{3}{*}{12.54}
&2&92.34&-\\
&&&&4&92.92&-\\
&&&&6&93.16&-\\
\cmidrule{3-7}
PLIF~\cite{fang2021incorporating}&Direct Training&6 Conv, 2 FC&36.71&8&93.50&-\\
\cmidrule{3-7}
\multirow{3}{*}{Dspike~\cite{li2021differentiable}}
&\multirow{3}{*}{Direct Training}
&\multirow{3}{*}{ResNet-18}
&\multirow{3}{*}{11.21}
&2&93.13&71.68\\
&&&&4&93.66&73.35\\
&&&&6&94.25&74.24\\
\midrule
\multirow{3}{*}{Spikformer~\cite{zhou2023spikformer}}
&\multirow{3}{*}{Direct Training}
&Spikformer-4-256&4.13&4&93.94&75.96\\
&&Spikformer-2-384&5.74&4&94.80&76.95\\
&&Spikformer-4-384&9.28&4&95.19&77.86\\
\cmidrule{3-7}
\multirow{3}{*}{Spikingformer~\cite{zhou2023spikingformer}}
&\multirow{3}{*}{Direct Training}
&Spikingformer-4-256&4.13&4&94.77&77.43\\
&&Spikingformer-2-384&5.74&4&95.22&78.34\\
&&Spikingformer-4-384&9.28&4&95.61&79.09\\
\cmidrule{3-7}
\multirow{1}{*}{Spike-driven Transformer~\cite{yao2023spike}}
&\multirow{1}{*}{Direct Training}
&Spike-driven Transformer-2-512&10.21&4&95.6&78.4\\
\cmidrule{3-7}
\multirow{1}{*}{\bf SpikingResformer (Ours)}
&\multirow{1}{*}{Direct Training}
&SpikingResformer-Ti\tnote{*}&10.79&4& 96.24& 79.28\\
\midrule
\multirow{3}{*}{Spikformer~\cite{zhou2023spikformer}}
&\multirow{3}{*}{Transfer Learning}
&Spikformer-4-384&9.28&4&95.54&79.96\\
&&Spikformer-8-384&16.36&4&96.64&82.09\\
&&Spikformer-8-512&29.08&4&97.03&83.83\\
\cmidrule{3-7}
\multirow{2}{*}{\bf SpikingResformer (Ours)}
&\multirow{2}{*}{Transfer Learning}
&SpikingResformer-Ti&10.76&4& 97.02& 84.53\\
&&SpikingResformer-S&17.25&4&\bf 97.40&\bf 85.98\\

\bottomrule
\end{tabular}

\begin{tablenotes}
\item[*] The stem structure differs from the original SpikingResformer-Ti in the main text.
\end{tablenotes}
\end{threeparttable}
}
\label{tab:transfer static}
\end{table*}

\begin{table*}[!t]
\centering
\caption{\centering{Detailed comparison on neuromorphic datasets.}}
\footnotesize
{
\begin{threeparttable}
\begin{tabular}{ccccccc}
\toprule
\multirow{2}{*}{\bf Method} & \multirow{2}{*}{\bf Type} & \multirow{2}{*}{\bf Archtecture } & \multirow{2}{*}{\#{\bf Param}~(M)} & \multirow{2}{*}{\bf T} & \multicolumn{2}{c}{{\bf Top-1 Acc.\ }(\%)}\\
\cmidrule{6-7}
&&&&&CIFAR10-DVS & DVSGesture\\
\midrule
\multirow{2}{*}{STBP-tdBN~\cite{zheng2021going}}
&\multirow{2}{*}{Direct Training}
&ResNet-19&12.54&10&67.8&-\\
&&ResNet-17&1.40&40&-&96.87\\
\cmidrule{3-7}
\multirow{2}{*}{PLIF~\cite{fang2021incorporating}}
&\multirow{2}{*}{Direct Training}
&5 Conv, 2 FC&17.22&20&74.8&-\\
&&6 Conv, 2 FC&1.69&20&-&97.57\\
\cmidrule{3-7}
Dspike~\cite{li2021differentiable}&Direct Training&ResNet-18&11.21&10&75.4&-\\
\midrule
\multirow{2}{*}{Spikformer~\cite{zhou2023spikformer}}&\multirow{2}{*}{Direct Training}&\multirow{2}{*}{Spikformer-2-256}&\multirow{2}{*}{2.55}&10&78.6&95.8\\
&&&&16&80.6&97.9\\
\cmidrule{3-7}
\multirow{2}{*}{Spikingformer~\cite{zhou2023spikingformer}}&\multirow{2}{*}{Direct Training}&\multirow{2}{*}{Spikingformer-2-256}&\multirow{2}{*}{2.55}&10&79.9&96.2\\
&&&&16&81.3&98.3\\
\cmidrule{3-7}
\multirow{1}{*}{Spike-driven Transformer~\cite{yao2023spike}}
&\multirow{1}{*}{Direct Training}
&Spike-driven Transformer-2-256&2.55&16&80.0&\bf 99.3\\
\cmidrule{3-7}
\multirow{2}{*}{\bf SpikingResformer (Ours)}
&\multirow{2}{*}{Direct Training}
&SpikingResformer-Ti\tnote{*}&10.79&10&81.5& -\\
&&SpikingResformer-XTi\tnote{*}&2.71&16& -& 98.6\\
\midrule
\multirow{2}{*}{\bf SpikingResformer (Ours)}
&\multirow{2}{*}{Transfer Learning}
&SpikingResformer-Ti&10.76&10& 84.7& 93.4\\
&&SpikingResformer-S&17.25&10&\bf 84.8& 93.4\\

\bottomrule
\end{tabular}

\begin{tablenotes}
\item[*] The stem structure differs from the original SpikingResformer-Ti in the main text.
\end{tablenotes}
\end{threeparttable}
}
\label{tab:transfer neuromorphic}
\end{table*}

\section{Further Comparison of Direct Training Results}
In the main text, we perform only transfer learning experiments on small-scale datasets including CIFAR10, CIFAR100, CIFAR10-DVS, and DVSGesture. In this section, we perform direct training experiments on these datasets.
To adapt to the small-size inputs, we replace the 7$\times$7 convolution with a 3$\times$3 convolution and remove the 3$\times$3 max pooling.
Other modules remain unchanged.
In addition, the minimum SpikingResformer-Ti is still too large for the DVSGesture dataset.
Therefore, we halved the number of channels in each layer to avoid overfitting, thus constructing SpikingResformer-XTi.
As shown in Tab.~\ref{tab:transfer static} and Tab.~\ref{tab:transfer neuromorphic}, our method is consistently effective when directly trained on small-scale datasets.
For instance, SpikingResformer-Ti achieves 96.24\% accuracy when directly trained on CIFAR10 dataset and achieves 79.28\% accuracy on CIFAR100 dataset, outperforming state-of-the-art method Spikingformer-4-384 by 0.63\% on CIFAR10 dataset and 0.19\% on CIFAR100 dataset.
For neuromorphic datasets, SpikingResformer-Ti achieves 81.5\% accuracy on CIFAR10-DVS dataset, outperforming Spikingformer-2-256 by 0.2\%. SpikingResformer-XTi achieves 98.6\% accuracy on DVSGesture dataset, outperforming Spikingformer-2-256 by 0.3\%, and competitive with 99.3\% accuracy of Spike-driven Transformer-2-256.

\section{Detailed Comparison of Transfer Learning Results}

\noindent
\textbf{Static Image Datasets.}
As shown in Tab.~\ref{tab:transfer static}, SpikingResformer outperforms other transfer learning methods on CIFAR10 and CIFAR100 datasets with fewer parameters.
The SpikingResformer-Ti achieves 84.53\% accuracy on the CIFAR100 dataset, outperforming Spikformer-8-512 by 0.7\% with only 11.14M parameters.
Moreover, the SpikingResformer-S achieves 85.98\% accuracy on the CIFAR100 dataset, which is the state-of-the-art result and outperforms Spikformer-8-512 by 2.15\%.
Compared to direct training methods, the SpikingResformer obtained from transfer learning has significantly higher performance.
The SpikingResformer-Ti outperforms the Spiking Transformer-2-512 by 6.1\% with a comparable number of parameters.
This demonstrates the advantage of transfer learning.

\noindent
\textbf{Neuromorphic Datasets.}
Since the existing spiking vision transformer does not perform transfer learning experiments on neuromorphic datasets, we mainly compare with direct training methods.
However, since the size of the models trained directly on the CIFAR10-DVS and DVSGesture is typically much smaller than the models pre-trained on ImageNet, we are not able to compare them to models with comparable parameters.
As shown in Tab.~\ref{tab:transfer neuromorphic}, the SpikingResformer obtained from transfer learning has significantly higher performance on CIFAR10-DVS.
The SpikingResformer-Ti achieves 84.7\% accuracy on CIFAR10-DVS, outperforming Spiking Transformer-2-256 by 4.7\% and outperforming Spikingformer-2-256 by 3.4\%.
However, the transfer learning results on DVSGesture fail to achieve comparable performance to direct training.
SpikingResformer only achieves 93.4\% accuracy on DVSGesture, falling behind the state-of-the-art method Spike-driven Transformer by 5.9\%.
We believe that this is mainly due to the way CIFAR10-DVS is constructed differs from DVSGesture.
CIFAR10-DVS is converted from CIFAR10 using a dynamic vision sensor, which does not contain temporal information.
Thus, models pre-trained on static datasets can transfer to CIFAR10-DVS well.
However, DVSGesture is directly created from human gestures, which contain rich temporal information. As a result, models pre-trained on static datasets do not transfer well to DVSGesture.

\end{document}